\documentclass[conference]{IEEEtran}
\usepackage{times}

\usepackage[numbers]{natbib}
\usepackage{multicol}
\usepackage[bookmarks=true]{hyperref}
\usepackage{amsmath}
\usepackage{amsthm}
\usepackage{macros}
\usepackage[english]{babel}
\usepackage{amsthm}
\usepackage{graphicx, array, blindtext}
\usepackage{dblfloatfix}
\theoremstyle{definition}
\newtheorem{definition}{Definition}
\newtheorem{theorem}{Theorem}
\newtheorem{proposition}{Proposition}

\usepackage{soul}
\usepackage{algorithm}
\usepackage{algpseudocode}
\usepackage{subcaption}
\usepackage{newfloat}
\usepackage{adjustbox}
\usepackage{listings}
\usepackage{xcolor}
\definecolor{mydarkblue}{rgb}{0,0.08,0.45}
\usepackage{tikz}
\usepackage{mathtools} 
\usepackage{amsfonts} %
\usepackage{pgfplots}
\pgfplotsset{compat=1.18}
\usetikzlibrary[pgfplots.groupplots]
\usepackage{wrapfig}

\usetikzlibrary{arrows, automata, positioning}
\tikzset{auto, >=stealth}
\tikzset{every edge/.append style={shorten >= 1pt}}
\tikzset{
    main node/.style={circle,draw,minimum size=1cm,inner sep=0pt},
}

\DeclareCaptionStyle{ruled}{labelfont=normalfont,labelsep=colon,strut=off} 
\floatstyle{ruled}
\newfloat{listing}{tb}{lst}{}
\floatname{listing}{Listing}
\usepackage{cleveref}
\Crefname{algocfline}{Algorithm}{Algorithms}
\Crefname{algocf}{line}{lines}

\usepackage{xcolor}

\begin{document}

\title{Real-Time Privacy Preservation \\ for Robot Visual Perception}


\author{\authorblockN{Minkyu Choi$^{*}$, Yunhao Yang$^{*}$, Neel P. Bhatt$^{*}$, Kushagra Gupta, Sahil Shah, Aditya Rai, \\ David Fridovich-Keil, Ufuk Topcu, Sandeep P. Chinchali}
\authorblockA{The University of Texas at Austin\\
Austin, Texas, USA}
}



%

\maketitle

\IEEEpeerreviewmaketitle
\begin{abstract} 
Many robots (e.g., iRobot's Roomba) operate based on visual observations from live video streams, and such observations may inadvertently include privacy-sensitive objects, such as personal identifiers.
Existing approaches for preserving privacy rely on deep learning models, differential privacy, or cryptography. They lack guarantees for the complete concealment of all sensitive objects. Guaranteeing concealment requires post-processing techniques and thus is inadequate for real-time video streams. We develop a method for privacy-constrained video streaming, \texttt{PCVS}, that conceals sensitive objects within real-time video streams. \texttt{PCVS} takes a logical specification constraining the existence of privacy-sensitive objects, e.g., never show faces when a person exists. It uses a detection model to evaluate the existence of these objects in each incoming frame. Then, it blurs out a subset of objects such that the existence of the remaining objects satisfies the specification. We then propose a conformal prediction approach to (i) establish a theoretical lower bound on the probability of the existence of these objects in a sequence of frames satisfying the specification and (ii) update the bound with the arrival of each subsequent frame. 
Quantitative evaluations show that \texttt{PCVS} achieves over 95 percent specification satisfaction rate in multiple datasets, significantly outperforming other methods. The satisfaction rate is consistently above the theoretical bounds across all datasets, indicating that the established bounds hold. Additionally, we deploy \texttt{PCVS} on robots in real-time operation and show that the robots operate normally without being compromised when \texttt{PCVS} conceals objects.
\end{abstract}
\newcommand{\saftyproperty}{P_{\text{safe}}}
\newcommand{\spec}{\Phi}
\newcommand{\vlm}{\mathcal{M}_{vl}}
\newcommand{\prob}{\mathbb{P}}
\newcommand{\pgs}{$\mathcal{PG}_k(\spec)$}
\newcommand{\pcvs}{\texttt{PCVS}}

\def\thefootnote{*}\footnotetext{These authors contributed equally to this work.}
\section{Introduction}

While robots utilize visual observations from video streams during operational routines for decision-making purposes, recording and disseminating such videos potentially exposes private information \cite{HMDB}, raising ethical and legal concerns. These concerns include risks of the inadvertent capture of sensitive personal data, unauthorized access, and misuse of recorded footage. 
A recent story highlighting a Roomba taking images of a person in a toilet room attests to the legitimacy of privacy concerns during robotic operations \cite{roomba}.
\begin{figure}[!t]
    \includegraphics[width=\linewidth]{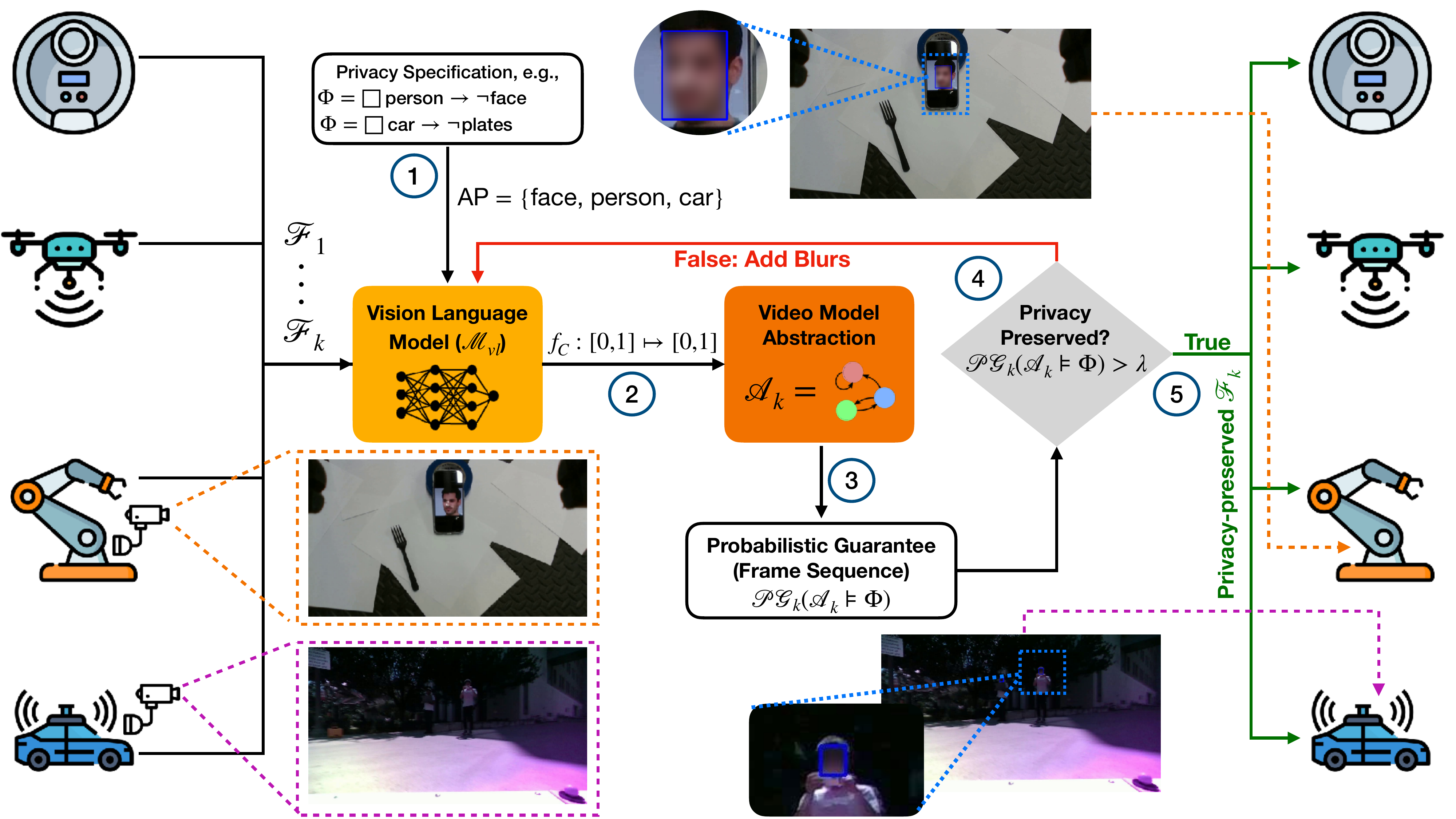}
    \caption{\textbf{Pipeline of Privacy-Constrained Video Streaming:} 
    \textbf{(Step 1)} Given a privacy specification $\Phi$, we define a set $AP$ of atomic propositions describing privacy-sensitive objects. 
    \textbf{(Step 2)} Given an incoming frame $\mathcal{F}_k$ from the video, the method uses a vision-language model (VLM) to detect sensitive objects in the frame. Each detection is associated with a confidence score from the VLM. The method calibrates a confidence score to a per-frame probability bound for correct detection via a calibration function $f_C$, as in Equation \ref{eq: calibration}.
    \textbf{(Step 3)} The method builds an abstract model $\mathcal{A}_k$ representing object detections and their probability bounds in the frame sequence $\mathcal{F}_1,...,\mathcal{F}_k$ via Algorithm \ref{algo: rt-abstract}. Then, it computes a theoretical bound for the probability of $\mathcal{A}_k$ satisfying $\Phi$, i.e., a probabilistic guarantee $\mathcal{PG}_k(\mathcal{A}_k \models \Phi)$ using Equation \ref{eq: sequence-guarantee}.
    \textbf{(Step 4)} If $\mathcal{PG}_k(\mathcal{A}_k \models \Phi)$ is below a user-given privacy threshold $\lambda$, the method removes a subset of sensitive objects and goes back to Step 2 to recompute a guarantee.
    \textbf{(Step 5)} If $\mathcal{PG}_k(\mathcal{A}_k \models \Phi)$ is above $\lambda$, the method adds $\mathcal{F}_k$ back to the stream and proceeds to Step 1 with the next incoming frame.
    We number each step in blue.
    }
    \label{fig:teaser_figure}
\end{figure}

Existing approaches protect privacy by concealing sensitive objects, but they either fail to guarantee complete concealment or cannot process real-time video streams.
While concealing sensitive objects requires detecting and locating such objects, existing works rely on deep-learning models for object detection \cite{padmanabhan2023gemel, sugianto2024privacy, kagan2023zooming}. However, due to their black-box nature, deep-learning models cannot provide theoretical guarantees on the correctness of the detection results. On the other hand, formal methods techniques, such as model checking, can guarantee that a given video adheres to privacy concerns \cite{umili2022grounding, yang2023specification, choi2024neuro}. However, the computational complexity of formal methods techniques grows with the video length, hence incapable of real-time video streams.



We develop a method to conceal privacy-sensitive objects in real-time video streams from robot cameras. The method takes a logical specification constraining the existence of sensitive objects. The specifications allow users to describe complex privacy requirements with conjunctions, disjunctions, implications, etc. For each incoming frame from the video streams, the method first uses a vision-language model to detect and locate all sensitive objects. Next, it removes a subset of objects (add Gaussian blurs or blackout) so that the existence of the remaining objects satisfies the specification.

We then establish a theoretical bound on the probability of complete concealment of sensitive objects in a video stream.
As deep learning models are typically over-confident in detecting objects, we use conformal prediction to calibrate the model's confidence to a probability of correct detection. 
Next, we express the specification as a temporal logic formula and build a finite automaton representing the object detections in a sequence of frames and the probabilities of correct detections. Then, we compute the probability that this automaton satisfies the specification (i.e., the video frames encountered so far preserve privacy). 

We develop a video abstraction algorithm that allows us to optimize the computational complexity involved with the arrival of each subsequent frame from the video stream. This abstraction is key to our method achieving
real-time performance, updating the probability with each frame arrival. This probability acts as a metric and helps users determine whether to use the video based on their privacy tolerance.

We evaluate the method over two large-scale datasets and present real-robot examples for real-time privacy protection. The method achieves 80 to 97 percent specification satisfaction rates in various scenarios, significantly outperforming existing automated solutions. Meanwhile, the method preserves all non-sensitive information. By seamlessly integrating concealment capabilities into the robot's visual perception system, we prevent potential privacy leakage from the robot. Simultaneously, this integration ensures the unhindered functionality of the robot's control policies, enabling it to operate normally without compromise.

\section{Related Work}

Privacy preservation in real-time video analytics has been the focus of several recent methods \cite{padmanabhan2023gemel, sugianto2024privacy, kagan2023zooming, chu2013real, wickramasuriya2005privacy, wang2017scalable, neff2019revamp, yuan2020minor, upmanyu2009efficient}. However, they rely solely on deep learning models for object detection, i.e., detecting and blurring privacy-sensitive entities in video. Due to the black-box nature of neural networks, these methods lack a quantitative guarantee.

To this end, formal verification approaches have guaranteed that a given complete video adheres to privacy safety concerns formulated as temporal logical specifications. For example, recent works \cite{umili2022grounding, yang2023specification, choi2024neuro, cheng2014temporal, sharan2024neurosymbolicevaluationtexttovideomodels, yang2024fine} construct a finite automaton representing video frame sequences and verify this automaton against temporal logical specifications. However, their approaches do not account for uncertainties related to the vision-based detection algorithms \cite{bhatt2024know}. Moreover, the construction and verification of automatons cannot be done in real-time.

On the other hand, some works using differential privacy or cryptography do not rely on deep learning models and, hence, can provide theoretical guarantees. For instance, Cangialosi et al. \cite{cangialosi2022privid} developed a differential privacy mechanism to protect video privacy, and Rahman et al. \cite{cryptography1} propose a cryptographic approach for video privacy. However, without integrating with deep learning models, these methods cannot interpret and enforce complex privacy requirements. In contrast, our method enforces the video satisfying any complex privacy requirements expressed in logic formulas.

\section{Problem Formulation}

A video \(\mathcal{V}\)  is a sequence of frames $\mathcal{F}_1, ...,\mathcal{F}_k$ where each $\mathcal{F}_k \in \mathbb{R}^{C \times  W\times H}$ is an RGB image with $C$ channels, $W$ width, and $H$ height. A video can be prerecorded or live-streamed from sources such as autonomous vehicles or security cameras.

We define a \textbf{privacy specification} \(\spec\) as a temporal logic formula \cite{rescher2012temporal} constraining the appearance of privacy-sensitive objects. 
Since we want to preserve privacy at all times, we express a privacy specification as $\Phi = \Box(\Tilde{\Phi})$, where $\Box$ represents the ``ALWAYS" temporal operation and $\Tilde{\Phi}$ is a first-order logic formula \cite{barwise1977introduction}. The presence of privacy-sensitive objects is constrained by \(\spec\).

We define a set of atomic propositions $AP$, where each proposition $p_i \in AP$ is a textual description of a privacy-sensitive object. 
Then, we use a vision-language model (VLM), \(\vlm\), to detect these objects.
\(\vlm : \mathbb{R}^{C \times  W\times H} \times AP \rightarrow [0, 1]\) takes a frame $\mathcal{F}_k \in \mathbb{R}^{C \times  W\times H}$ and a proposition $p_i \in AP$ as inputs, and returns a confidence score $c \in [0,1]$, denoted as $c = \vlm (\mathcal{F}_k, p_i)$. However, deep learning models are often overconfident, and their detection accuracy cannot be guaranteed. 
Therefore, we calibrate the confidence using conformal prediction \cite{conf-pred}, which provides a lower bound for the probability of correctly detecting privacy-sensitive objects in every frame, considering the inherent uncertainty in deep learning model predictions.

However, traditional conformal prediction approaches focus on post-processing and do not account for temporal events.
Therefore, we use calibrated confidence to detect and constrain privacy-sensitive objects over time and provide a probabilistic guarantee on a sequence of frames.

To achieve this, we develop an algorithm $\mathcal{VA}$ that takes a sequence of $k$ frames and returns a formally verifiable video abstraction $\mathcal{A}_k$ encoding the object detection across the sequence: $\mathcal{VA}([\mathcal{F}_1, ...,\mathcal{F}_k]) = \mathcal{A}_k$. The \textbf{video abstraction} $\mathcal{A}_k$ is represented as a labeled Markov chain, detailed rigorously in Section \ref{sec: method} as it requires extensive background and mathematical notation. This provides a probabilistic guarantee on a frame sequence via formal verification \cite{formalmethod}.

\begin{definition}[\textbf{Probabilistic Guarantee on a Frame Sequence}]\label{def:video_sequence_guarantee}
Given a sequence of frames $\mathcal{F}_1, ...,\mathcal{F}_k$, a privacy specification \(\spec\), and a video abstraction \( \mathcal{A}_k \) at the $k^{\text{th}}$ frame, a probabilistic guarantee $\mathcal{PG}_k(\mathcal{A}_k \models \spec)$ on the frame sequence $\mathcal{F}_1, ..., \mathcal{F}_k$ represents the theoretical minimum probability that the presence of privacy-sensitive objects in the frame $\mathcal{F}_1$ through $\mathcal{F}_k$ adheres to \(\spec\).
\end{definition}

\textbf{Problem 1 (Real-Time Video Privacy Preservation).} 
Given a frame sequence $\mathcal{F}_1,...,\mathcal{F}_{k-1}$, an incoming frame $\mathcal{F}_k$ from a video stream, a privacy specification $\spec$, and an algorithm $\mathcal{VA}$ that builds a video abstraction from the frame sequence, we aim to remove privacy-sensitive objects in $\mathcal{F}_k$ such that $\mathcal{A}_k = \mathcal{VA}([\mathcal{F}_1, ...,\mathcal{F}_k])$ satisfies $\spec$ with a probability at least $\mathcal{PG}_k(\mathcal{A}_k \models \spec)$. 




\section{Privacy-Constrained Video Streaming}\label{sec: real-time-privacy}
\label{sec: method}

We develop privacy-constrained video streaming (\texttt{PCVS}), a method to enforce live video streams that satisfy a user-given privacy specification with a probabilistic guarantee. 
The overall pipeline for \texttt{PCVS} is illustrated in Figure \ref{fig:teaser_figure}.

\paragraph{Real-Time Video Privacy Preservation Framework} 
We explain our framework with a running example in a real-time video stream from a real robot (see \Cref{fig: example}). We aim to hide human faces so that no personal identity will be revealed in vision-based robot operations.
Therefore, the privacy specification is \(\spec = \lalways \text{person} \rightarrow \neg \text{face}\), where $\rightarrow$ and $\neg$ mean ``implies'' and ``not,'' respectively.
We detect humans and faces at every frame via the VLM.
Subsequently, we use conformal prediction to obtain a calibrated confidence score for the detection in the current frame. We build a video abstraction $\mathcal{A}_k$ to represent the detection results for humans and faces across a sequence of frames and utilize it to obtain a probabilistic guarantee on \(\spec\) being satisfied. 
We then verify if the guarantee is above the user-given privacy threshold $\lambda \in [0, 1]$. 
If this threshold is not met, we iteratively remove the detected faces and update the guarantee $\mathcal{PG}_k(\mathcal{A}_k \models \spec)$ until the threshold is met.
\subsection{Probabilistic Guarantee on Video Privacy}

Given a sequence of $k$ frames and a privacy specification $\Phi$, we compute a probabilistic guarantee $\mathcal{PG}_k(\mathcal{A}_k \models \spec)$. This guarantee is updated at each incoming frame. We use formal methods to prove that the guarantee holds.

\begin{figure}[t]
  \centering
    \includegraphics[width=\linewidth]{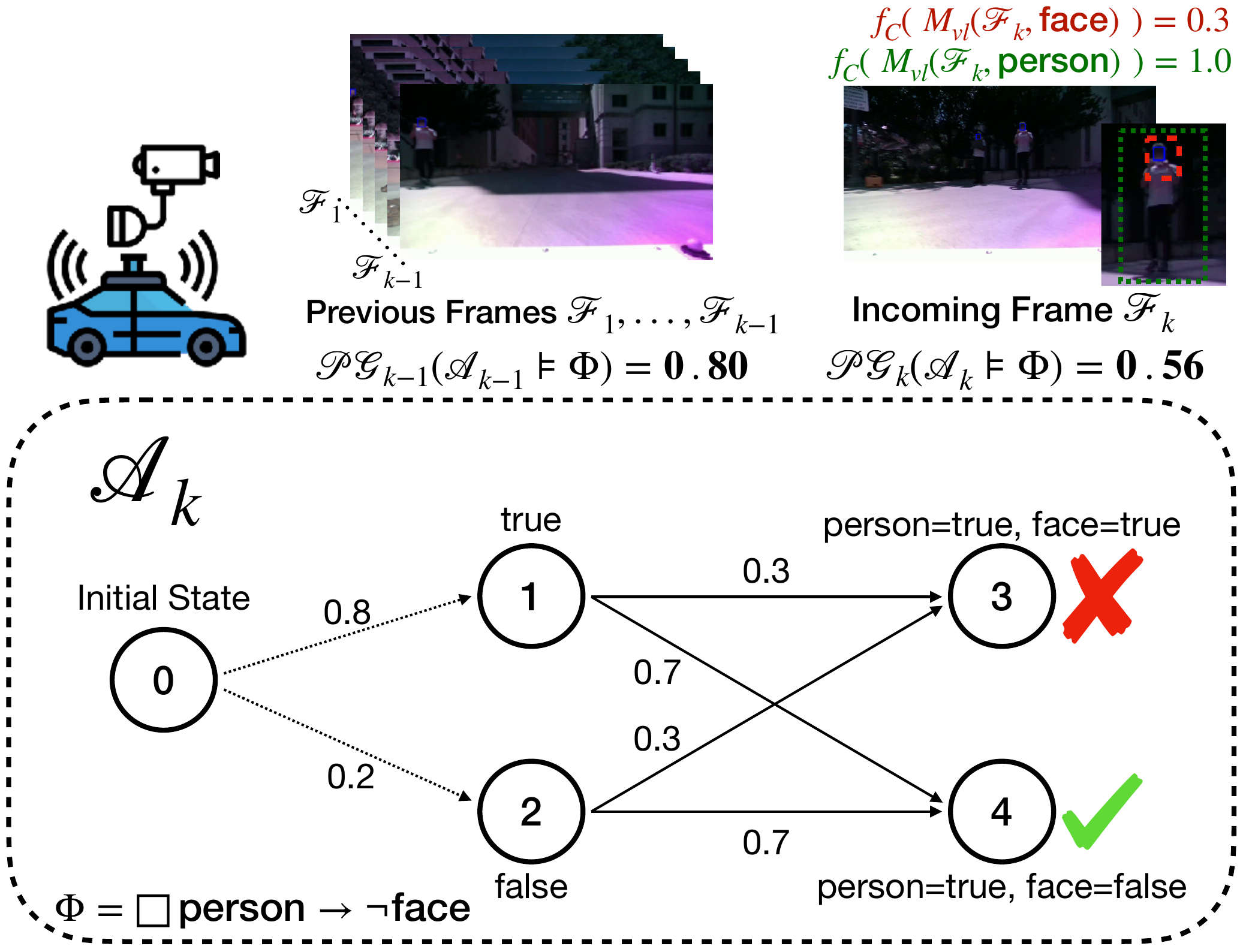}
    \caption{A running example on how to compute the probabilistic guarantee via video abstraction.}
    \label{fig: example}
\end{figure}
\paragraph{Confidence Calibration via Conformal Prediction}
Recall that a VLM \(\vlm(x_i,y_i) = c\) receives an image $x_i$ and a textual object label $y_i$ as a prompt and returns a confidence score $c \in [0, 1]$. Given the VLM \(\vlm\) and a labeled calibration dataset that is distributed identically to the task domain, using conformal prediction \cite{conf-pred}, we learn a calibration function $f_C : [0,1] \mapsto [0,1]$ that maps a confidence score, $c \in [0,1]$ to a lower bound for the probability of correct detection. 

We first collect a calibration set $\{(x_i, y_i)\}_{i=1}^m$ consisting of $m$ (image, ground truth text label) tuples. 
Then, we apply $\vlm$ to detect the privacy-sensitive objects in the images $\{x_i\}_{i=1}^m$ and get a set of \emph{nonconformity scores}: $\{1 - \vlm(x_i, y_i) \}_{i=1}^m$. A nonconformity score is the sum of confidence scores of wrong detections. Next, we estimate a probability density function of these scores, denoted as $f_{nc}(z)$, where $z$ is a nonconformity score. Then, we use Theorem \ref{thm: CP} to establish a theoretical lower bound for the probability of the correct detection.

\begin{theorem}
\label{thm: CP}
    Let $\epsilon \in [0, 1]$ be a pre-defined error bound and $x_{n}$ be an image outside the calibration set. We define a \emph{prediction band} as $\hat{C}(x_{n}) = \{p_i : \vlm(x_{n}, p_i) \ge 1 - c^*, p_i \in AP \}$. Then, according to conformal prediction, there exists a confidence $c^*$ such that $\epsilon = 1 - \int_0^{c^*} f_{nc}(z) dz$  satisfies $\mathbb{P}\left[y_{n} \in \hat{C}(x_{n}) \right] \ge 1 - \epsilon$, where $y_{n}$ is the ground truth label for $x_{n}$. Proof in \cite{conf-pred}. 
\end{theorem}

Note that $\vlm (x_i, p_i)$ returns a single confidence score indicating whether $p_i$ exists in $x_i$. By the theory of conformal prediction, $1 - \epsilon$ is a theoretical lower bound for the probability of the ground truth label belonging to the prediction band $\hat{C}(x_{n})$. 
If $\vlm (x_i, p_i) > 0.5$, we provide a lower bound for the probability of the existence of $p_i$. Otherwise, if $\vlm(x_i, p_i) \leq 0.5,$ we bound the probability of non-existence.
Hence, we get a calibration function 
\begin{equation}
\label{eq: calibration}
    f_C (c) = \begin{cases}
        \int_0^{c} f_{nc}(z) dz, \text{ if } c > 0.5 \\
        \int_0^{1 - c} f_{nc}(z) dz, \text{ otherwise.} 
    \end{cases}
\end{equation}

\paragraph{Video Abstraction.} 

For verifying a \emph{real-time} video stream against the privacy specification \(\spec\), a key challenge is to verify the temporal behaviors of \emph{all} the previously received frames plus the current frame. This makes verification space- and time-inefficient because we must repeatedly verify previous frames for each new incoming frame. To overcome this challenge, we build an abstraction for the video stream, which enables real-time verification.

\begin{definition}[\textbf{Video Abstraction}]\label{def: dtmc} A video abstraction is a labeled Markov chain \((S, s_0, P, L)\), where $S$ is a set of states, each state corresponds to a conjunction of atomic propositions, \(s_0 \in S\) is the initial state, $P: S\times S\rightarrow [0,1]$ is a transition function. $P(s, s')$ represents the probability of transition from a state $s$ to a state $s'$ and $\sum_{s' \in S}P(s, s') = 1$. $L:S\rightarrow 2^{AP}$ is a labelling function.
\end{definition}

\begin{algorithm}
    \caption{Real-Time Video Abstraction}
    \label{algo: rt-abstract}
    \begin{algorithmic}[1]
    \Require $\text{vision-language model }\vlm, \text{calibration function} f_C,$ 
    \Statex $\text{propositions set }AP,\text{\,specification }\,\Phi, \text{ probability }p_{k-1}$
    \Statex $\text{of previous frames satisfying}\Phi, \text{incoming frame}\, \mathcal{F}_{k}$
    
    \State $S_{\mathrm{obs}}, P, L$ = \{\}, \{\}, \{\} \Comment{Initialize the abstraction}
    \State $S_{\mathrm{obs}}$.add(0), $S_{\mathrm{obs}}$.add(1), $S_{\mathrm{obs}}$.add(2) \Comment{We represent each state with an Arabic numeral}
    \State $L(1) = \mathrm{false}$, $L(2) = \mathrm{true}$ 
    

    \State $P(0, 1) = 1 - p_{k-1}$,  $P(0, 2) = p_{k-1}$ \Comment{Add transitions to indicate the probability of previous frames satisfying $\Phi$}

    \State $i = 3$ \Comment{Initialize a indexer representing states}

    \For{$\sigma$ in $2^{AP}$} \Comment{$\sigma$ is a conjunction of atomic propositions}
        \State prob = $\prod_{p \in \sigma} f_C(\vlm (\mathcal{F}_{k}, p))$ \Comment{Get a lower bound for a detection result}
        \If{ prob $> 0$ } \Comment{Add a state to represent the detection with the lower bound}
            \State $S_{\mathrm{obs}}$.add($i$), $L(i) = \sigma$, $P(1, i) =$ prob, $P(2, i) =$ prob, $i = i + 1$
        \EndIf
    \EndFor
    
    \end{algorithmic}
    \Return{$S_{\mathrm{obs}}, s_0=0, P, L$}
\end{algorithm}

We propose Algorithm \ref{algo: rt-abstract} to build video abstractions. We demonstrate it through an example in Figure \ref{fig: example}. First, we add an initial state (State 0 in Figure \ref{fig: example}), a state representing the event that all previous frames (if they exist) satisfy $\Phi$ (State 1), and a state representing the event that previous frames fail $\Phi$ (State 2), as in lines 1-3.
Next, we add transitions from State 0 to State 1 and to State 2 with the probability of previous frames satisfying \(\spec\) as in line 4. Then, we detect objects in the incoming frame $\mathcal{F}_k$ and get the probability bound for correct detection. For each conjunction of propositions (e.g., person=true and face=false), we build a state and add transitions to this state with the probability bound of correctly detecting objects described in this conjunction, as in lines 6-9. Hence, we obtain the video abstraction $\mathcal{A}_k$ (e.g., Figure \ref{fig: example}).

Following Algorithm \ref{algo: rt-abstract}, we incrementally add new states to the abstraction (rather than build a new one) with the arrival of each new incoming frame and check it against $\Phi$. Hence, this abstraction can be used to check video streams efficiently. Then, we theoretically prove that the probabilistic guarantee obtained through this abstraction holds.

\paragraph{Probabilistic Guarantees for Frame Sequence}
Given a video abstraction \(\mathcal{A} = (S, s_0, P, L)\), we define a \emph{path} \(\pi\) as a sequence of states starting from \(s_0\). The states evolve according to the transition function \(P\). A \emph{prefix} is a finite path fragment starting from \(s_0\). We define a \emph{trace} as $\psi = \text{trace}(\pi) = L(s_0)L(s_1)L(s_2)\dots$, where $s_0, s_1, s_2 ,... \in \pi$. Traces($\mathcal{A}$) denotes the set of all traces from $\mathcal{A}$. Each trace $\psi = L(s_0)L(s_1)L(s_2)\dots$ is associated with a probability $\mathbb{P}(\psi) = P(s_0, s_1) \times P(s_1, s_2) \times$...

The privacy specification is in the form of $\lalways \Tilde{\Phi}$. Hence, a privacy specification describes a \emph{safety property} \cite{baier2008principles}.
\begin{definition}[\textbf{Safety Property}]
\label{def: safe}
    A safety property $\psafe$ is a set of traces in $(2^{AP})^\omega$ ($\omega$ indicates infinite repetitions) such that for all traces $\psi \in (2^{AP})^\omega \backslash \psafe$, there exists a finite prefix $\hat{\psi}$ of $\psi$ such that
    $
        \psafe \cap \{\psi'\in(2^{AP})^\omega |\, \hat{\psi} \, \mathrm{is\, a\, prefix\, of\, }\psi'\} = \Phi.
    $
    $\hat{\psi}$ is a \emph{bad prefix} and $\mathrm{BadPref}(\psafe)$ is the set of all bad prefixes with respect to \(\psafe\).
\end{definition}
A video satisfies the privacy specification if its abstract representation $\mathcal{A}$ satisfies the safety property \(\psafe\), i.e., Traces$(\mathcal{A})\subseteq \psafe$. 
The probability that a video satisfies the specification is
\begin{equation} \label{eq: safe-prob}
\begin{split}
    \mathbb{P}[\mathcal{A}\mathrm{\, is\, safe}] &= \mathbb{P}[\pi\in\mathrm{\,path}(s_0) \,|\, \mathrm{trace}(\pi)\in \psafe]\\ 
    &= \sum_{\psi \in \mathrm{Traces}(\mathcal{A})\cap \psafe} \mathbb{P}(\psi).
\end{split}
\end{equation}
Note that this probability is a probabilistic guarantee on a sequence of frames. According to the definition of safety property, we derive the following theorem:
\begin{theorem}\label{thm1}
    Consider a set of prefixes $\hat{\Psi}$ such that $\mathbb{P}\{\hat{\psi}\in\hat{\Psi} \,|\,\hat{\psi} \notin \mathrm{BadPref}(\psafe)\}\geq \alpha$. Let $\Bar{S}\subset S$ be a subset of states in $\mathcal{A}$ such that
    $\mathbb{P}\{\hat{\psi}L(s)\notin \mathrm{BadPref}(\psafe) \,|\,\hat{\psi} \notin \mathrm{BadPref}(\psafe)\mathrm{\,and\,} s \in \Bar{S}\}\geq \beta$. Then,
$\mathbb{P}\{\hat{\psi}L(s)\notin \mathrm{BadPref}(\psafe) \,|\,s \in \Bar{S} \mathrm{\,and\,} \hat{\psi}\in \hat{\Psi}\}\geq \alpha \beta.$ 
\end{theorem}

\begin{proof}
    Let $A = \{\hat{\psi}L(s) \notin \mathrm{BadPref}(\psafe) \,|\,s \in \Bar{S} \cap \hat{\psi}\in \hat{\Psi}\} $ and $B = \{\hat{\psi}\notin \mathrm{BadPref}(\psafe) \, | \, \hat{\psi}\in \hat{\Psi}\}$. Then, $A|B = \{\hat{\psi}L(s) \notin \mathrm{BadPref}(\psafe) \,|\,s \in \Bar{S} \cap \hat{\psi}L(s) \notin \mathrm{BadPref}(\psafe)\}$ and $\mathbb{P}(A) = \mathbb{P}(A|B)\cdot \mathbb{P}(B) \geq \alpha\beta$.
\end{proof}

From Theorem \ref{thm1}, we can compute a new probabilistic guarantee on a sequence of frames after each incoming frame. However, the length of the abstraction's prefixes increases as the stream continues, leading to high complexity. Therefore, we derive the following proposition to show that Theorem \ref{thm1} holds even if we fix the length of the prefixes (proof of the proposition is in the Appendix):
\begin{proposition}
    \label{prop1}
    Let $\hat{\psi}_T$ and $\hat{\psi}_F$ be single element prefixes whose corresponding paths only consist of one state such that $\hat{\psi}_T \notin \mathrm{BadPref}(\psafe)$ and $\hat{\psi}_F \in \mathrm{BadPref}(\psafe)$. Let $\mathbb{P}(\hat{\psi}_T) = \alpha$, $\mathbb{P}(\hat{\psi}_F) = 1-\alpha$, $\Psi'=\{\psi_T, \psi_F\}$, then if we replace $\hat{\Psi}$ with $\Psi'$ in \Cref{thm1}, the Theorem still holds. 
\end{proposition}

\begin{proof}
    Since \(\mathbb{P}(\hat{\psi}_T) = \alpha\) and \(\mathbb{P}(\hat{\psi}_F) = 1-\alpha\), \(\prob\{\hat{\psi} \in \Psi' | \hat{\psi} \notin \mathrm{BadPref}(\psafe)\} = \alpha\). The replacement of \(\hat{\Psi}\) with \(\Psi'\) does not affect \(\mathbb{P}\{\hat{\psi}L(s)\notin \mathrm{BadPref}(\psafe) \,|\,\hat{\psi} \notin \mathrm{BadPref}(\psafe)\cap s \in \Bar{S}\}\).
    Thus, the conditions of Theorem 2 are satisfied, and \(\mathbb{P}\{\hat{\psi}L(s)\notin \mathrm{BadPref}(\psafe) \,| s \in \Bar{S} \cap \hat{\psi}\in \Psi'\} \geq \alpha \beta\). 
\end{proof}

From \Cref{thm1} and \Cref{prop1}, we can compute $\mathcal{PG}_k(\mathcal{A}_k \models \Phi)$ as follows:
\begin{equation}
    \label{eq: sequence-guarantee}
\begin{split}
    \mathcal{PG}_k(\mathcal{A}_k \models \Phi) &= \mathcal{PG}_{k-1}(\mathcal{A}_{k-1} \models \Phi) \\
    &\times \left( \sum_{\sigma \models \Tilde{\Phi}} \prod_{p \in \sigma} f_C(\vlm (\mathcal{F}_{k}, p)) \right)
\end{split}
\end{equation}

The video abstraction captures all previous frames in only two states (States 1 and 2 in Figure \ref{fig: example}) instead of accumulating states for every frame in the sequence. Thus, we can efficiently update the guarantee through a single computation. In Figure \ref{fig: example}, $\mathcal{PG}_k(\mathcal{A}_k \models \Phi) = 0.8 \times 0.7 = \mathbf{0.56}$.

\section{Robot Demonstrations}
\begin{figure*}[!t]
    \centering
    \includegraphics[width=\textwidth]{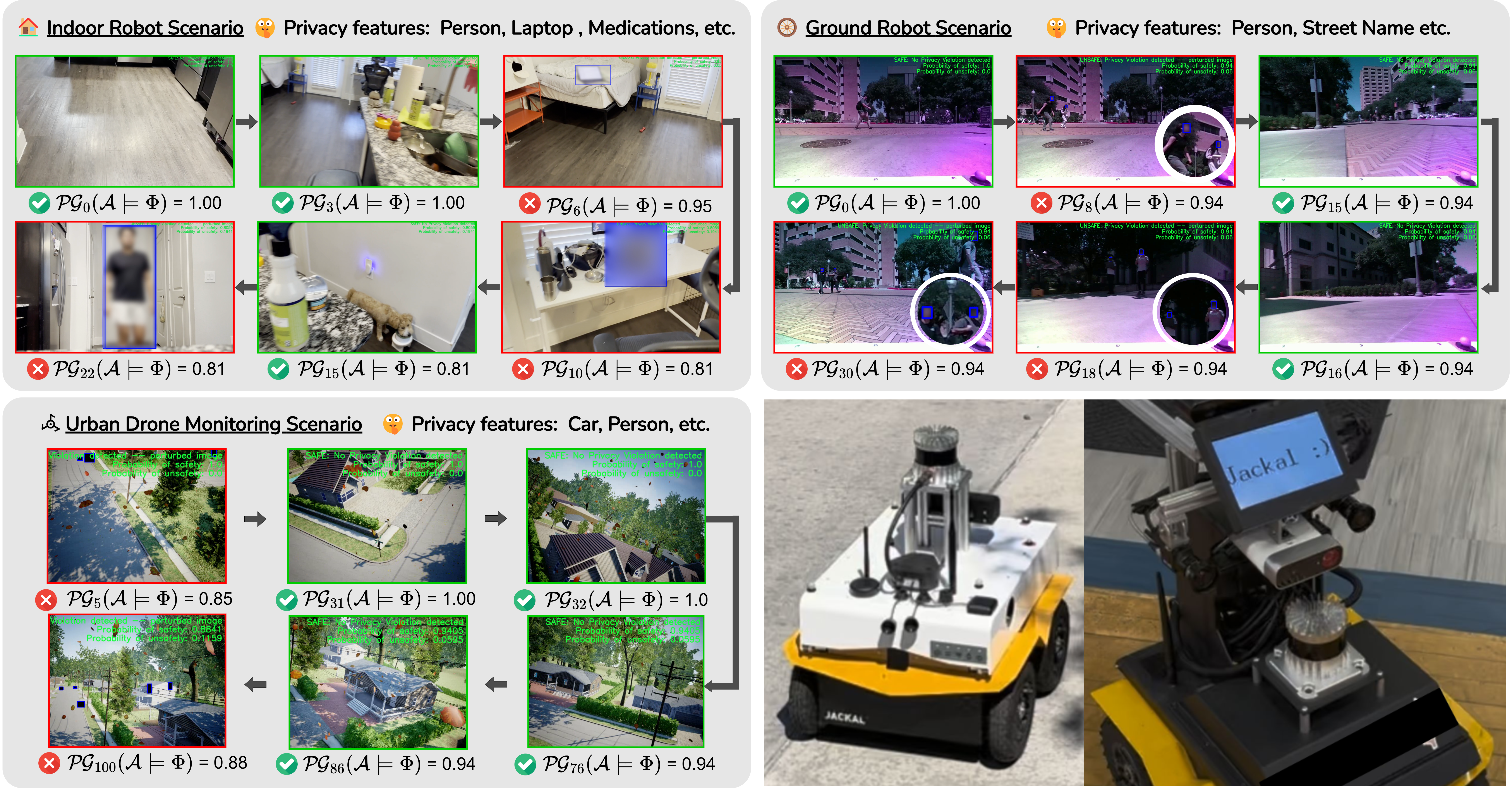}
    \caption{We present the demonstrations for indoor robot navigation (top left), ground robot navigation (top right), and urban drone monitoring (bottom left). The indoor robot and the ground robot are shown in the bottom right. Scenes with an `x' in a red circle contain privacy-sensitive objects and our method successfully conceals them.
    All demonstrations effectively maintain privacy above the user-given privacy threshold of 0.80, denoted as $\mathcal{PG}_k(\spec) > 0.80$.
    }
    \label{fig:real_robot_demonstration}
\end{figure*}

Our experiments assess our method in two areas: (i) its ability to protect privacy, and (ii) its efficiency in preserving other features for vision-based robot tasks.


We demonstrate our approach on a Jackal ground robot for autonomous driving, an indoor robot for in-house navigation and service, and a drone for urban monitoring (see \Cref{fig:real_robot_demonstration}). 
Given video streams from robot cameras, we aim to execute actions based on the control policy ($\pi$). Our approach effectively preserves privacy with formal guarantees without compromising performance in the real-time robot operation. We use YOLOv9 \cite{wang2024yolov9} in our method for all demonstrations. 

\paragraph{Indoor Navigation} 
In the first demonstration, we deploy \texttt{PCVS} to an indoor navigation robot to protect user privacy. We ground the robot in a private residence for in-house services such as transporting objects and house cleaning. While the robot perceives the environment through visual observations, we aim to preserve the privacy within such observations. The privacy specification is
\begin{equation*}
    \Phi_1 = \Box~(\neg~\text{laptop} \wedge \neg ~\text{medication} \wedge \neg ~\text{person}),
\end{equation*}
which requires hiding all people, laptops, and medications appearing in the scenes.

\Cref{fig:real_robot_demonstration} (top left) demonstrates how our method performs to conceal the sensitive objects such that the video satisfies the privacy specification. The probabilistic guarantee of privacy preservation throughout the complete operation is 0.81.

\begin{figure}
  \centering
    \begin{tikzpicture}[thick,scale=.6, node distance=2.2cm, every node/.style={transform shape}]

    \node[state,initial] (10) at (1, 0) {\Large $q_0$};
    \node[state] (11) at (0, 3) {\Large $q_{1}$};
    \node[state] (12) at (3, 3) {\Large $q_{2}$};
    \node[state] (13) at (6, 0) {\Large $q_{3}$};

     \draw[->, shorten >=1pt, sloped] (10) to[bend left] node[below, align=center] {\small reach \\ intersection} (11);

     \draw[->, shorten >=1pt, sloped] (11) to[bend left] node[above, align=center] {\small stop $\land$ observe} (12);

    \draw[->, shorten >=1pt, sloped] (12) to[left] node[below, align=center] {\small $\neg$ car $\land \neg $ person, \\ turn right} (13);
    \draw[->, shorten >=1pt, sloped] (12) to[loop above] node[above, align=center] {\small car $\vee$ person, \\ stop} ();
    \draw[->, shorten >=1pt, sloped] (13) to[loop above] node[above, align=center] {\small car $\vee$ person, \\ stop} ();
    \draw[->, shorten >=1pt, sloped] (13) to[left] node[above, align=center] {\small move forward} (10);
\end{tikzpicture}
    \caption{A sample control policy for the ground robot. Each transition is associated with an (input, output) tuple.}
    \vspace{-10pt}
    \label{fig:car_policy}
\end{figure}
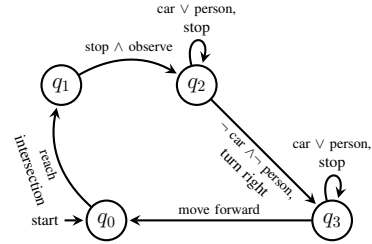

\paragraph{Ground Robot Driving} We deploy the control policy on the ground robot for five driving tasks, such as turning right at a stop sign (as presented in Figure \ref{fig:car_policy}). We embed \texttt{PCVS} in the robot's camera to conceal sensitive objects during real-time operation. 
In the driving example, we define a set of privacy specifications:
\begin{equation*}
    \begin{split}
\Phi_2 &= \Box~\neg ~\text{road sign},\\
\Phi_3 &= \Box~((\text{bicycle} \rightarrow \neg ~\text{person}) \vee (\text{person}\rightarrow \neg ~\text{face})),\\
\Phi_4 &= \Box~((\text{bus} \vee \text{car})\rightarrow \neg~\text{plate}).
    \end{split}
\end{equation*}
Intuitively, we want to conceal privacy-sensitive objects such as road name signs, car plates, and human faces. \Cref{fig:real_robot_demonstration} (top right) presents an example of how our method conceals sensitive objects to satisfy the specifications. The driving operation has a probabilistic guarantee of privacy preservation at 0.84, i.e., at least 84 percent of satisfying the specifications.

When concealing sensitive objects, it is crucial to ensure that such concealment does not adversely impact the robot’s decision-making processes. For instance, the robot should still be capable of detecting and avoiding pedestrians even after their faces have been obscured. More precisely, consider a safety specification:
\begin{equation*}
    \Phi_5 = \Box~((\text{car}\vee \text{person})\rightarrow \mathbf{X}~\text{stop},
\end{equation*}
which necessitates that the robot comes to a stop if cars or pedestrians are present ahead.

In the demonstration, the robot performs identically regardless of whether our method is deployed, and in both cases, it satisfies the safety specifications. Hence, we show that our privacy protection will not over-conceal non-sensitive objects and negatively impact the decision-making procedure in driving scenarios.

\begin{figure}[t]
    \centering\subcaptionbox{\label{fig:eval_privacy_preservation_by_length}}
    {
      \includegraphics[width=0.49\textwidth]{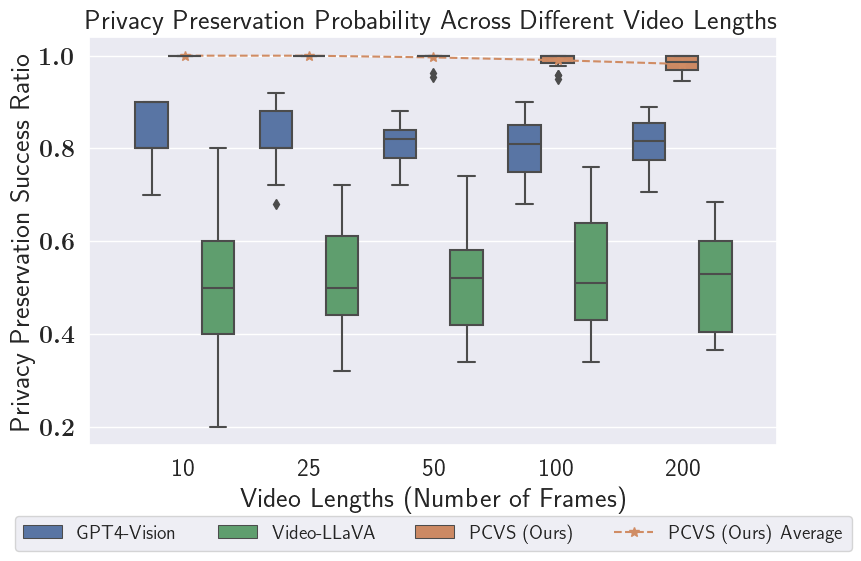}%
    }\vspace{-0.0em}  
    \subcaptionbox{\label{fig:eval_privacy_preservation_by_complexity}}
    {
      \includegraphics[width=0.48\textwidth]{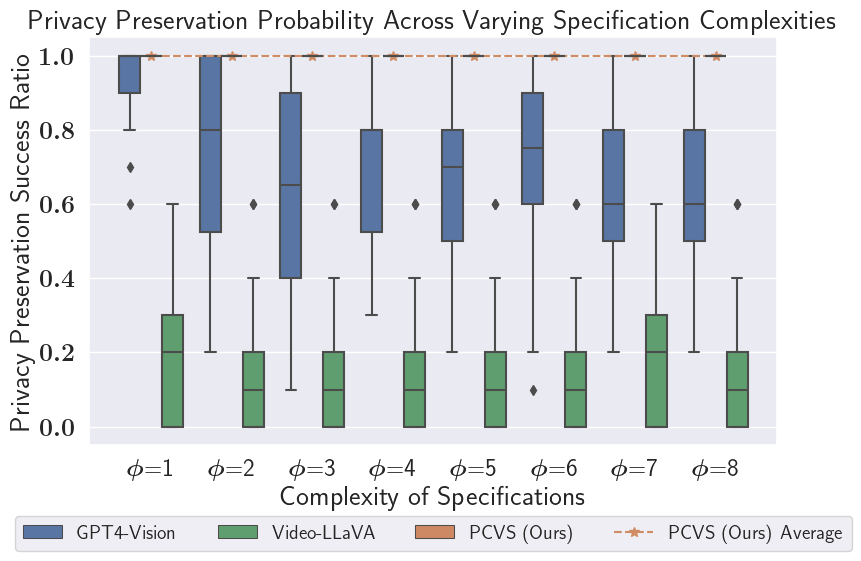}%
    }
    \caption{
    \textbf{\texttt{PCVS} effectively maintains privacy in long-horizon videos and complex privacy specifications}. In Figure \ref{fig:eval_privacy_preservation_by_length}, \texttt{PCVS} consistently preserves privacy, achieving an average Privacy Preservation Success Ratio of $0.97$ across various video lengths. In Figure \ref{fig:eval_privacy_preservation_by_complexity}, we show that \texttt{PCVS} consistently upholds privacy regardless of the complexity of specifications with an average Privacy Preservation Success Ratio of $0.94$.}
    \label{fig:eval_privacy_preservation}
    \vspace{-15pt}
\end{figure}

\paragraph{Urban Drone Monitoring} We demonstrate the applicability and effectiveness of our method in urban drone monitoring scenarios. The privacy specification is:
\begin{equation*}
    \Phi_6 = \Box~((\text{bicycle}\rightarrow \neg \text{person}) \wedge (\text{car} \vee \text{bus} \rightarrow \neg ~\text{person})),
\end{equation*}
which requires hiding all cars, buses, bicycles, and persons appearing in the scenes. 
\Cref{fig:real_robot_demonstration} (bottom left) demonstrates the successful performance of our method in concealing objects that are irrelevant to the monitoring task, yet require privacy preservation. The demonstration also indicates the real-time capability of our method to be seamlessly migrated to real-world drone applications.

\begin{figure}[t]
    \centering
    \includegraphics[width=\linewidth]{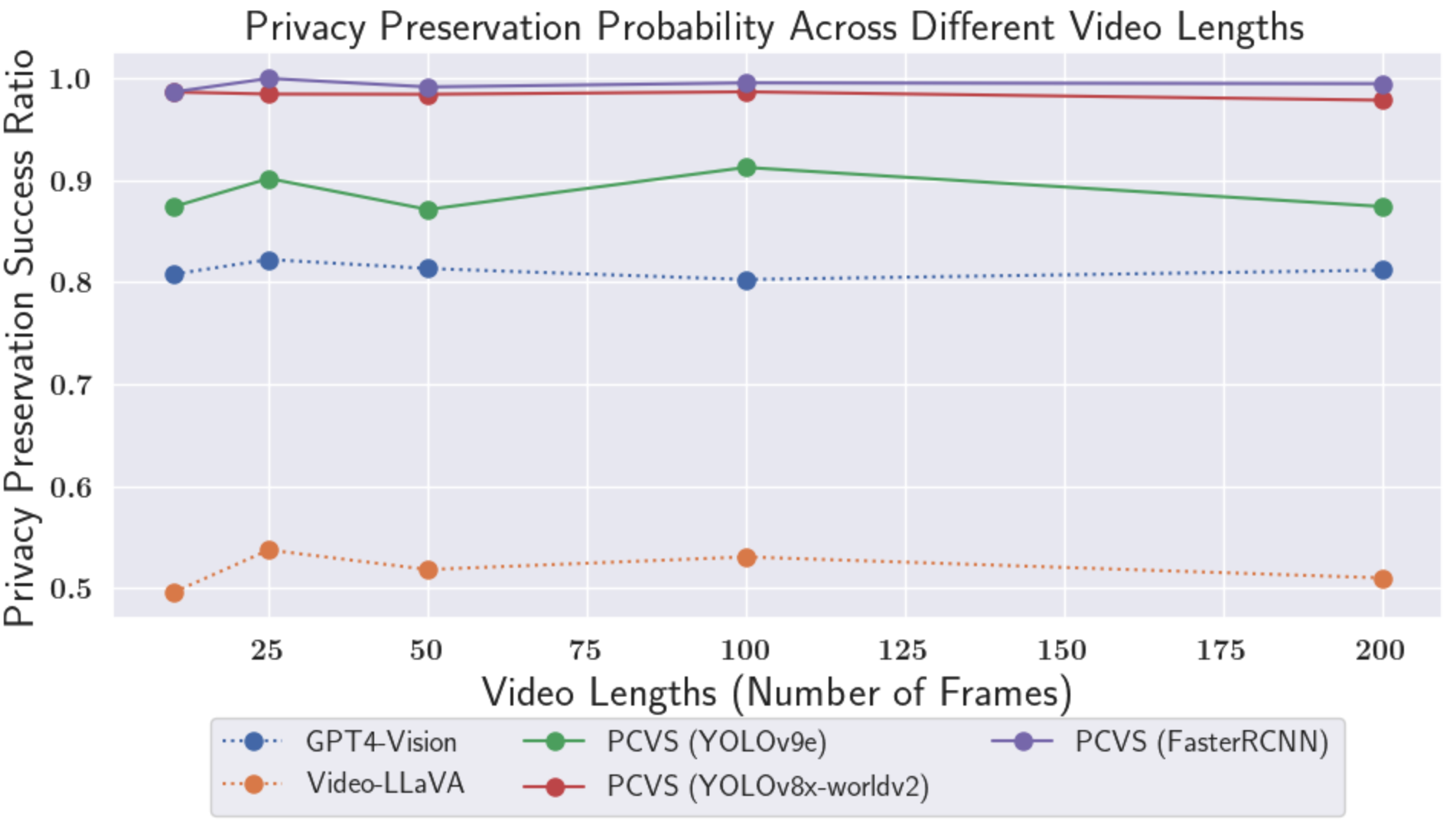}
    \caption{The comparisons between our method and other benchmarks and the comparisons between our method under different object detection models.}
    \label{fig: benchmarks}
\end{figure}

\section{Quantitative Analyses}

We present quantitative analyses in two areas: preserving privacy and preserving non-private visual features. We use YOLOv9 \cite{wang2024yolov9} on large-scale image datasets---ImageNet \cite{deng2009imagenet} and MS COCO \cite{cocodataset}---and a real-world driving dataset---UFPR-ALPR \cite{laroca2018robust}.

Our analyses show \texttt{PCVS} can preserve privacy even for long-horizon videos with complex privacy specifications. We define the \textbf{complexity of specifications}, $\phi$, as the \emph{number of propositions} in $\Phi$.
For instance, the complexity of a specification \( \Phi = \Box(\neg p_1 \wedge \neg p_2 \vee \neg p_3) \) with propositions \( AP = \{ p_1, p_2, p_3 \} \) is $\phi = 3$.

\textit{Evaluation Dataset I (ED 1)}\label{def:ed1}: We focus on the presence of ``person" in videos. We select images of a person from the ImageNet dataset and randomly insert these images at various positions for each video duration, filling any remaining slots with random images. 
We produce five different video lengths: 10, 25, 50, 100, and 200, with 25 video samples for each duration, resulting in 125 video samples overall.

\textit{Evaluation Dataset II (ED 2)}: We use the MS COCO dataset to evaluate our method at different complexities of specifications because it has multiple labels per image. Each level of specification complexity consists of 20 video samples, with an average of 50 frames per sample, resulting in a total of 160 videos. This dataset was developed using the same process as the ED1 dataset, with modifications made to accommodate the complexity of the specifications. For example, if $\phi = 3$, the privacy specification for the dataset is $\spec = \neg p_1 \wedge \neg p_2 \wedge \neg p_3$, where $p_1$, $p_2$, and $p_3$ are the ground truth labels of the selected image. These images are then randomly placed within the dataset, and the remaining slots are filled with random images. 

\textit{Evaluation Dataset III (ED 3)}: We randomly select a subset of images from the UFPR-ALPR dataset \cite{laroca2018robust} to generate videos with lengths 10, 25, 50, 100, and 200. Each length has 200 video samples (a total of 1000 videos). The dataset consists of labeled images that include driving-related objects such as vehicles, license plates, etc. We form a video by integrating a sequence of images from the dataset. 


\textit{Benchmarks}: To assess the ability of \texttt{PCVS} to detect privacy violations based on a given privacy specification, we use GPT-4 Vision \cite{openai2023gpt4vision} and Video LLaVA \cite{lin2023videollava} as benchmarks. This is because both benchmark methods can process a sequence of images from a video alongside a privacy specification. 

\begin{figure}[t]
    \centering
    \includegraphics[width=\linewidth]{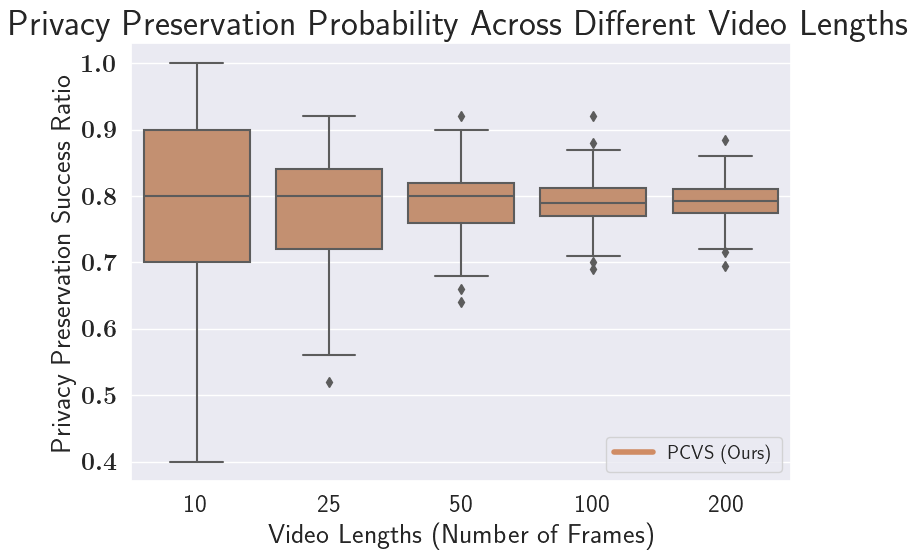}
    \caption{\textbf{Effective privacy preservation in real-world driving scenes.} We apply \texttt{PCVS} to driving scenes from the UFPR-ALPR dataset and conceal privacy-sensitive objects such as license plates. The privacy preservation success ratio of \texttt{PCVS} is consistently above 0.8.}
    \label{fig: privacy-preserving-alpr}
\end{figure}

\begin{figure}[t]
    \centering
    \includegraphics[width=0.45\linewidth]{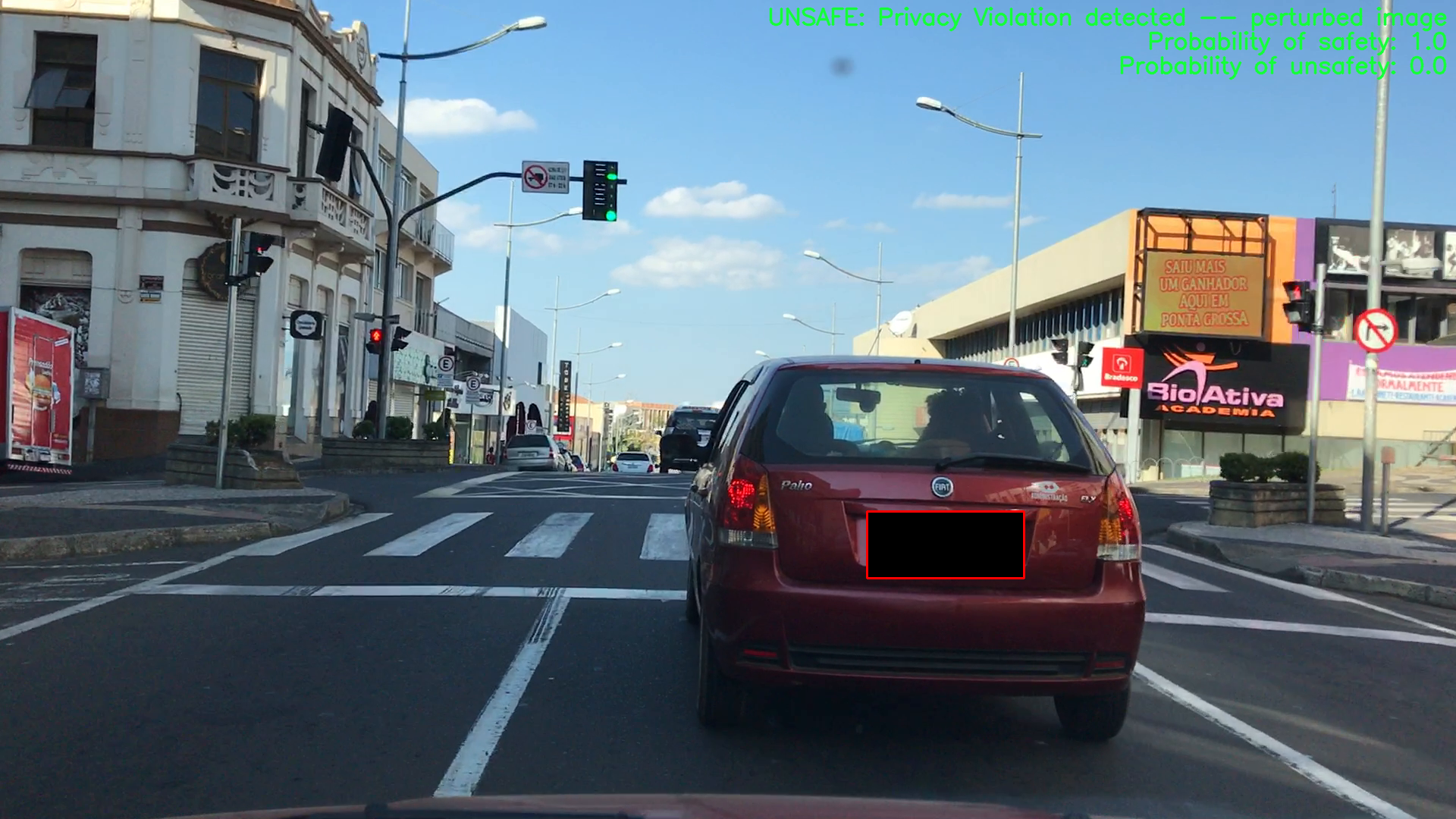}
    \includegraphics[width=0.45\linewidth]{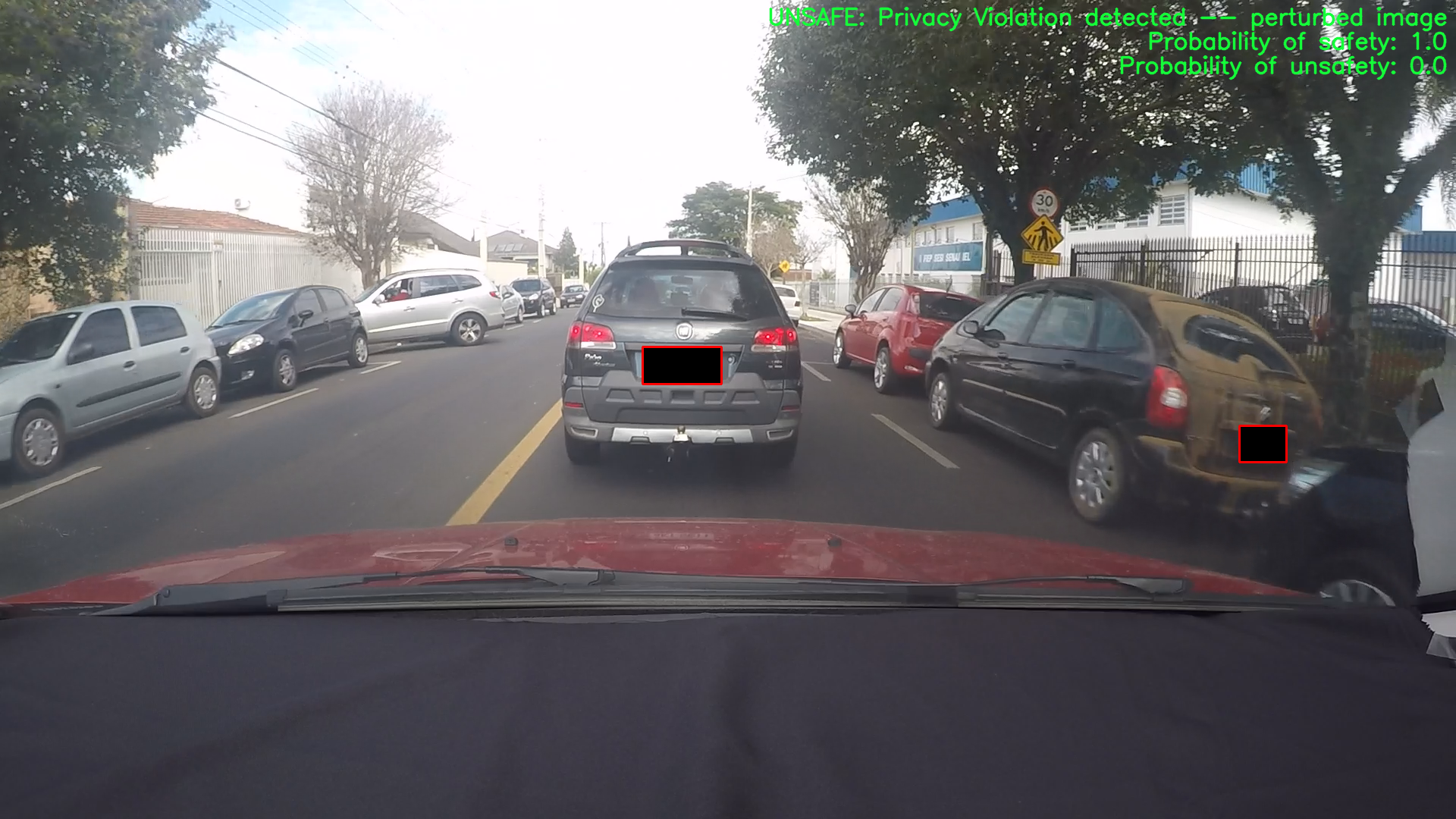}
    \includegraphics[width=0.45\linewidth]{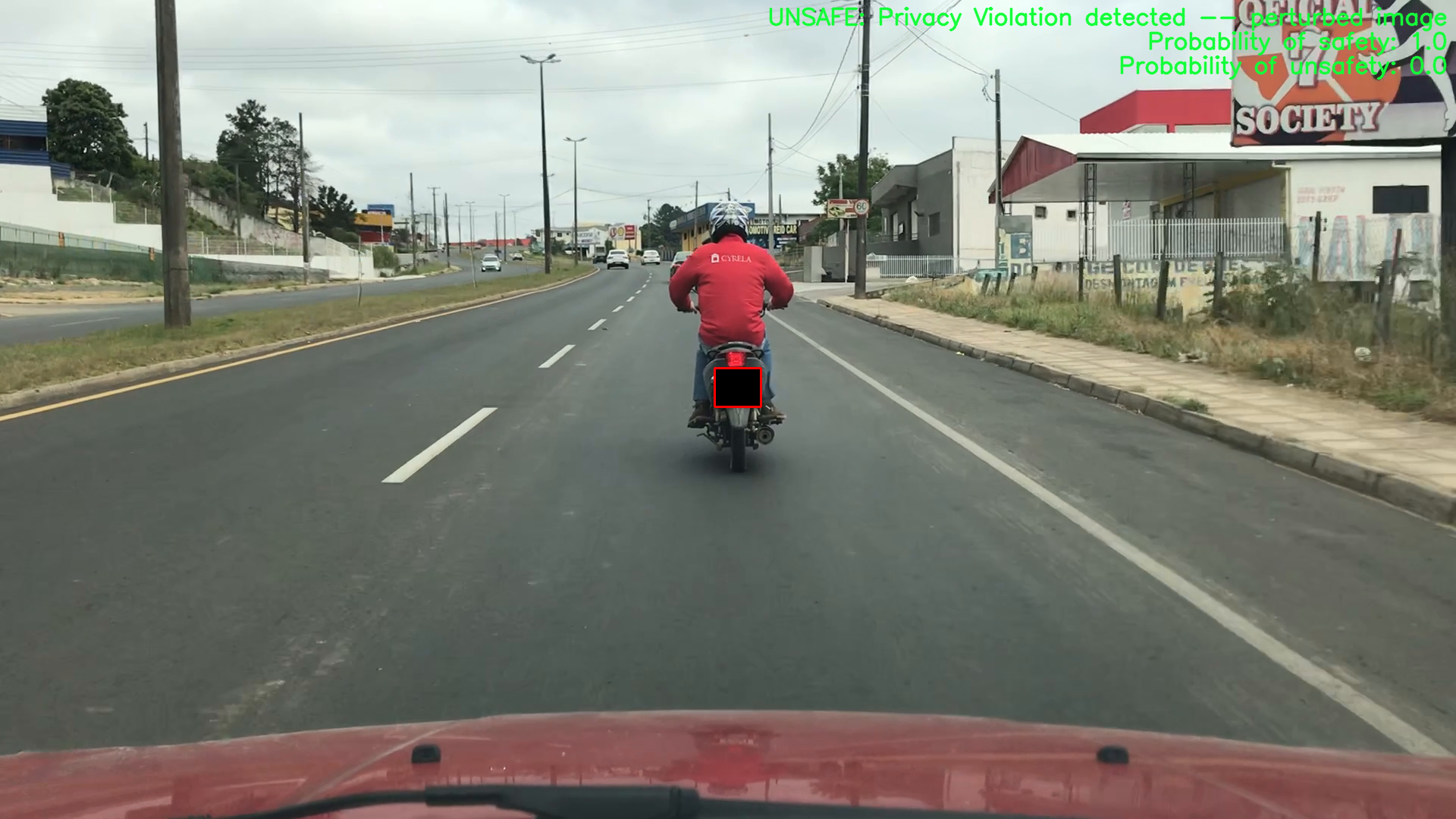}
    \includegraphics[width=0.45\linewidth]{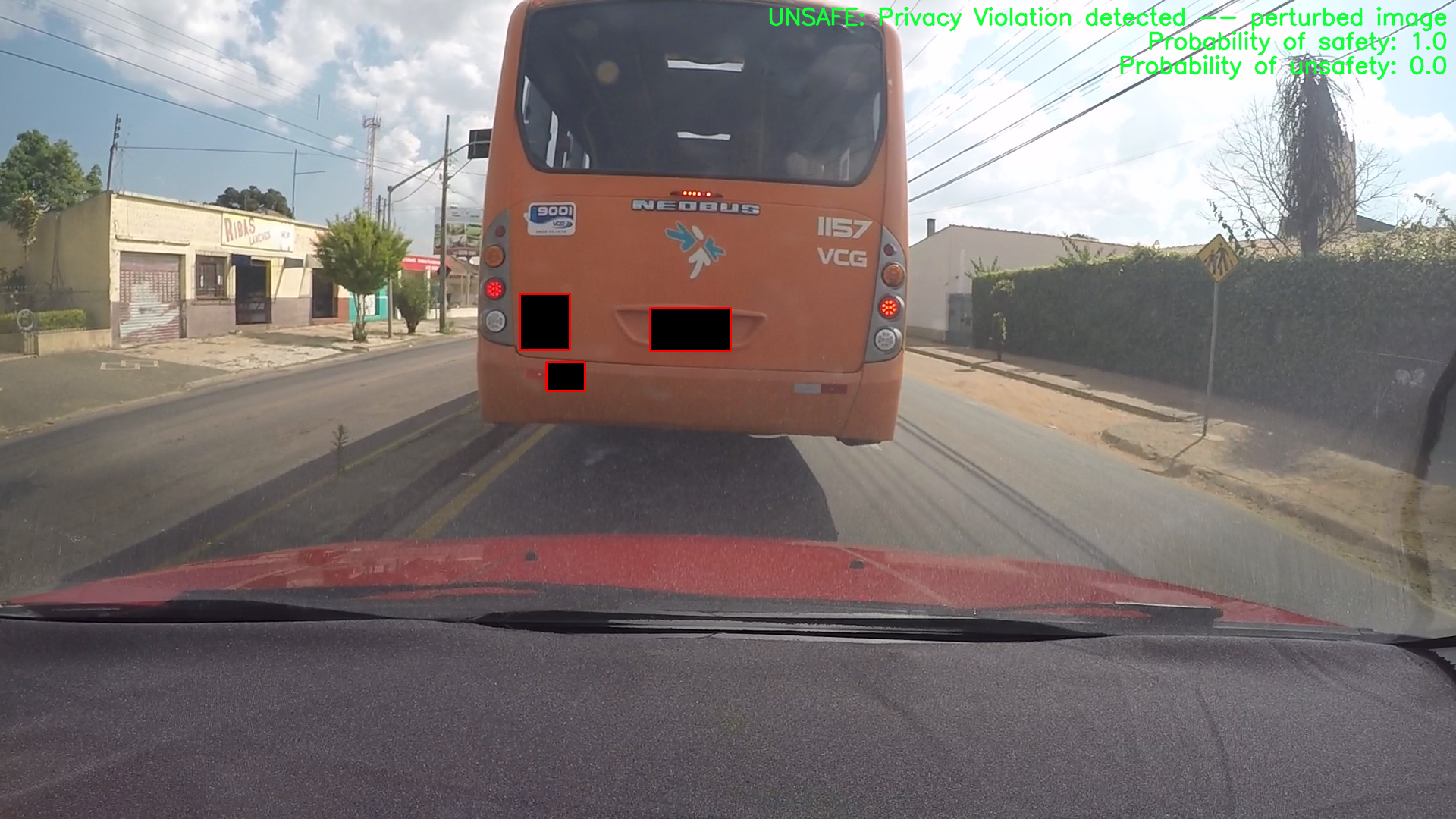}
    \caption{Demonstrations on our method concealing license plates in the driving scenes from the UFPR-ALPR dataset.}
    \label{fig: driving-data-demo}
    \vspace{-10pt}
\end{figure}


\subsection{Privacy Preservation}
We quantitatively evaluate the performance of privacy preservation across varying video lengths and specification complexities. For this evaluation, we define a metric representing the success ratio of privacy preservation as follows:
\begin{equation*}
\begin{split}
    \text{Privacy }&\text{Preservation Success Ratio} =\\ 
    &\frac{\text{Number of}\ p_i \in AP\ \text{detected or concealed}}{\text{Total number of private}\ p_i \in AP\ \text{within}\ \mathcal{V}}.
\end{split}
\end{equation*}
\texttt{PCVS} counts the number of concealed objects, while the benchmark counts the number of detected objects, as the benchmark does not natively conceal those objects. We assume that if those objects are detected, they can be concealed by other methods such as Gaussian concealing.

\paragraph{Comparison by video length} 
Our method ensures privacy preservation in live video streams, which means that the video length can be infinite. Hence, it is crucial to assess whether privacy is maintained as video streams become longer. To this end, we test \texttt{PCVS} on videos with various lengths from ED 1. We find that \texttt{PCVS} consistently maintains performance in preserving privacy, in contrast to benchmark methods that exhibit degraded performance as the length of videos increases (see Figure \ref{fig:eval_privacy_preservation_by_length}).

We also examine how the underlying vision model ability affects the privacy preservation ratio. We repeat the experiment on ED 1 while using different detection models: Yolov9e \cite{wang2024yolov9}, Yolov8x-worldv2 \cite{cheng2024yolow}, and FasterRCNN \cite{ren2016faster}, all with default parameters. We present the privacy preservation success ratio of our method using different detection models versus video lengths in Figure \ref{fig: benchmarks}. The results show that our method is sensitive to the detection model quality. Our method outperforms the benchmarks (GPT4-Vision and Video-LLaVA) at every video length when using the mainstream detection models.

To further demonstrate the real-world applicability of our method, we apply it to ED 3 and evaluate the privacy preservation success ratio across different video lengths. Recall that ED 3 consists of images collected from vehicle dash cameras. Figure \ref{fig: privacy-preserving-alpr} shows our method's high privacy preservation success ratio---consistently above 0.8 regardless of length. We present some demonstration figures in Figure \ref{fig: driving-data-demo}. The results indicate the applicability of our method to real-world tasks such as autonomous driving.

\begin{figure}
  \centering
    \includegraphics[width=\linewidth]{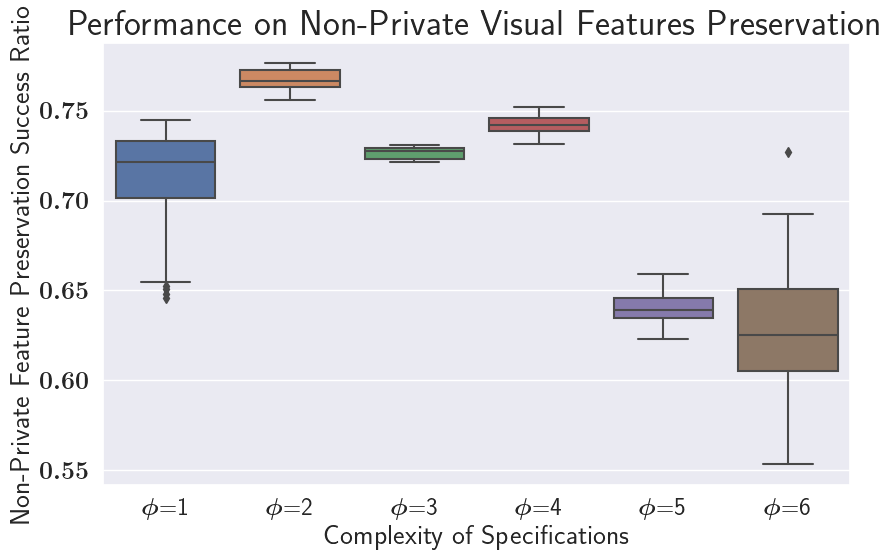}
    \caption{\textbf{Preserving non-private visual features for vision-based robot operation.} Our method can detect non-private objects after concealing private objects specified in $\spec$. However, performance degrades from $\phi=5$ because more private objects get concealed indirectly masking out other objects.}
    \vspace{-10pt}
    \label{fig:eval_non_privacy_preservation}
\end{figure} 

\paragraph{Comparison by specification complexities}
Next, we assess \texttt{PCVS} based on the complexity of specifications. This comparison is important because a privacy specification can be intricate, involving more than just two or three propositions. For example, a specification might require the detection and concealment of multiple privacy-sensitive objects within the same video, such as faces, license plates, and specific types of clothing. Our method significantly outperforms benchmark methods (see Figure \ref{fig:eval_privacy_preservation_by_complexity}) regardless of the complexity of specifications. We demonstrate that \texttt{PCVS} can effectively handle highly complex privacy compositions in real-time video streams, ensuring robust privacy protection.

\subsection{Non-private Visual Feature Preservation} 
Preserving non-private visual features is crucial for vision-based robot operation, as it relies on visual observation for control policies. In our demonstration (as presented in Figure \ref{fig:real_robot_demonstration}), the ground robot must be capable of identifying people from privacy-constrained video footage to make appropriate decisions, such as stopping. We analyze how our method preserves non-private visual features using ED2 in Figure \ref{fig:eval_non_privacy_preservation}. We define a metric that represents the success ratio of preserving non-private visual features as follows:
\begin{center}
    $\text{Non-Private Feature Preservation Success Ratio} = \frac{\text{Number of}\ \chi\ \text{detected after concealing} \ p_i \in AP}{\text{Total number of}\ \chi\ \text{within}\ \mathcal{V}},$
\end{center}
where $\chi$ is a non-private target object for detection. In our evaluation, non-private visual features remain preserved and detectable even after the concealment of privacy-sensitive objects as defined in the privacy specifications. However, the success ratio of non-private preservation decreases as the complexity of these specifications increases. This is because \texttt{PCVS} conceals a larger image area as the number of privacy-sensitive objects increases.

\begin{figure}[t]
    \centering
    \includegraphics[width=0.8\linewidth]{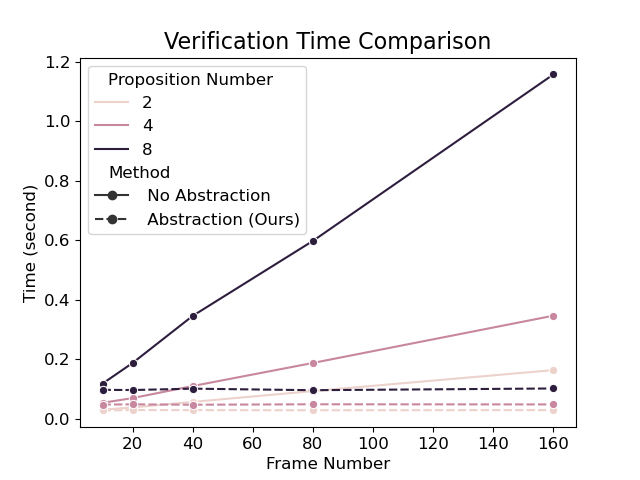}
    \caption{Comparison between the verification time with and without model abstraction. Our model abstraction significantly reduces the latency and remains constant latency when the proposition number increases.}
    \label{fig: verification-time}
    \vspace{-10pt}
\end{figure}

\subsection{Computational Complexity}
Model checking often incurs significant computational overhead, particularly as the state space size increases with the number of video frames, which limits its capability in real-time applications. Hence, we develop an abstraction method to resolve this limitation. 

Figure \ref{fig: verification-time} shows the verification time versus the frame number under different numbers of atomic propositions. The experiments are performed using the video collected by the ground robot under an \emph{Apple M2 CPU}. As the frame number increases, the verification time without our video abstraction method grows linearly, while the time with abstraction remains constant. 

Furthermore, we tested our method with YOLOv9 on both \emph{Intel Xeon gold} CPU and \emph{Nvidia A5000} GPU. The average runtimes for processing one proposition in one image frame (1600x900 pixels) using CPU and GPU are 159 milliseconds and 69 milliseconds. Therefore, a robot with a CPU can also preserve privacy in real-time at an approximate 6 frames per second (fps) frequency and a robot with a GPU is capable of videos with 14 fps. We present more details in Figure \ref{fig: step-time}.

\begin{figure}[t]
    \centering
    \includegraphics[width=0.75\linewidth]{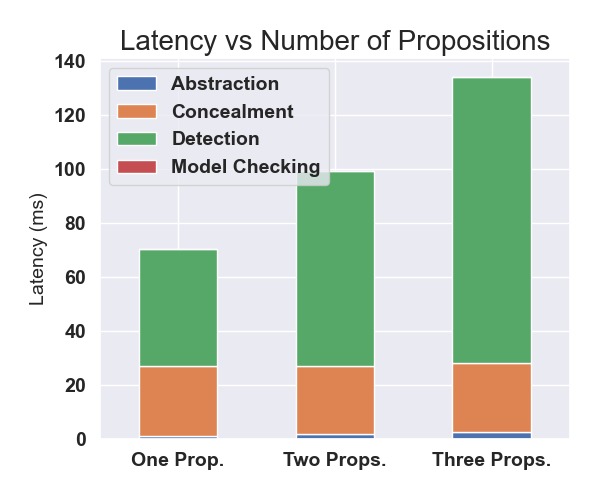}
    \caption{Latency comparison in each stage of our method to process one frame. Object detection takes the major proportion of the time and the time grows linearly as the number of propositions increases. The time for other stages is negligible or remains constant across different numbers of propositions.}
    \label{fig: step-time}
    \vspace{-10pt}
\end{figure}

\section{Conclusion}
\paragraph{Summary}
We propose \texttt{PCVS}, a method to protect privacy in a live video stream that is either generated from robotic tasks or fed to robot learning algorithms with a probabilistic guarantee of the privacy specification being satisfied. Our method significantly outperforms state-of-the-art methods for short and long video sequences. Further, we show the real-time capabilities of our method on three robotic applications. 

\paragraph{Limitations and Future Work} Our method is limited by the capabilities of the VLM that we use. We are currently unable to process specifications that are action-based in nature (e.g.,, remove humans who are seen eating in a video) because of the limited performance of the VLM (in the experiment) in action recognition. 
A future direction is to address this limitation by integrating models that process multiple frames at a time to detect whether an action has occurred. Another future direction is extending our privacy specifications to generic specifications that can describe more properties besides safety, such as liveness and fairness, improving the generalizability of our method.


\clearpage




\bibliographystyle{plainnat}
\bibliography{references}

\end{document}